\documentclass{article}
\usepackage{times}
\usepackage{soul}
\usepackage{url}
\usepackage[hidelinks]{hyperref}
\usepackage[utf8]{inputenc}
\usepackage[small]{caption}
\usepackage{graphicx}
\usepackage{amsmath}
\usepackage{booktabs}
\usepackage{xcolor}
\usepackage{algorithm}
\usepackage[noend]{algpseudocode}
\usepackage{amssymb}
\usepackage{amsthm}

\newtheorem{theorem}{Theorem}
\newtheorem{definition}{Definition}

\def \R {\mathbb R}
\def \Z {\mathbb Z}  
\def \N {\mathbb N}
\def \x {\mathbf{x}} % input variable
\def \lb {l} % lower bound
\def \ub {u} % upper bound
\def \w {{\mathbf w}} % hidden weights
\def \c {{\mathbf c}} % linear parameters

\algrenewcommand\algorithmicrequire{\textbf{Input}}
\algrenewcommand\algorithmicensure{\textbf{Output}}

\newcommand{\red}[1]{#1}

\begin{document}

\title{Black-box Combinatorial Optimization %with Costly and Noisy Evaluations
using Models with Integer-valued Minima}

% \author{
% First Author$^1$\footnote{Contact Author}\and
% Second Author$^2$\and
% Third Author$^{2,3}$\And
% Fourth Author$^4$\\
% \affiliations
% $^1$First Affiliation\\
% $^2$Second Affiliation\\
% $^3$Third Affiliation\\
% $^4$Fourth Affiliation\\
% \emails
% \{first, second\}@example.com,
% third@other.example.com,
% fourth@example.com
% }

%\author{Heuristic Search and Game Playing}
% \author{BLIND NAMES AND AFFILIATIONS\and
% Laurens Bliek\footnote{Contact Author} \and
% Sicco Verwer\and
% Mathijs de Weerdt\\
% \affiliations
% Delft University of Technology, Faculty of Electrical Engineering, Mathematics and Computer Science, Van Mourik Broekmanweg 6,
% 2628 XE Delft, the Netherlands\\
% \emails{l.bliek@tudelft.nl}
% }

\author{Laurens Bliek
\thanks{All authors at
              Delft University of Technology
              Faculty of Electrical Engineering, Mathematics and Computer Science\newline
              Van Mourik Broekmanweg 6
              \newline 2628 XE Delft
              \newline The Netherlands
              \newline 
              Email: l.bliek@tudelft.nl
%             \emph{Present address:} of F. Author  %  if needed
}%
, Sicco Verwer and Mathijs de Weerdt
}
%\author{Laurens Bliek}

% \institute{All authors \at
%               Delft University of Technology\\
%               Faculty of Electrical Engineering, Mathematics and Computer Science\\
%               Van Mourik Broekmanweg 6
%               \\2628 XE Delft
%               \\The Netherlands
%               \\
%               \email{l.bliek@tudelft.nl}           % 
% %             \emph{Present address:} of F. Author  %  if needed
% }

% \thanks{All authors at
%               Delft University of Technology}
              
%               Faculty of Electrical Engineering, Mathematics and Computer Science\\
%               Van Mourik Broekmanweg 6
%               \\2628 XE Delft
%               \\The Netherlands
%               \\
%               Email: l.bliek@tudelft.nl           % 
% %             \emph{Present address:} of F. Author  %  if needed
% }

%\date{Received: date / Accepted: date}

\maketitle

\begin{abstract}
%In many practical applications, an objective needs to be optimized for which no known mathematical formulation is known.
When a black-box optimization objective
%If this objective 
can only be evaluated with costly or noisy measurements, most standard optimization algorithms  are unsuited to find the optimal solution.
Specialized algorithms that deal with exactly this situation make use of surrogate models.
These models are usually continuous and smooth, which is beneficial for continuous optimization problems, but not necessarily for combinatorial problems.
However, by choosing the basis functions of the surrogate model in a certain way, we show that it can be guaranteed that the optimal solution of the surrogate model is integer.
%This approach is shown to outperform simulated annealing on a toy example and on a noisy traveling salesman benchmark problem.
This approach
%is shown to outperform simulated annealing on a toy example, and 
\red{outperforms random search, simulated annealing and one Bayesian optimization algorithm on the problem of finding robust routes for a noise-perturbed traveling salesman benchmark problem, with similar performance as another Bayesian optimization algorithm, and outperforms all compared algorithms on a convex binary optimization problem with a large number of variables.}
\end{abstract}

\section{Introduction}

Traditional optimization techniques make use of a known mathematical formulation of the objective function, for example by calculating the derivative or a lower bound.
However, many objective functions in real-life situations have no known mathematical formulation.
For example, smart grids or railways are complex networks where every decision influences the whole network.
%One example is choosing the number of nodes and layers in a deep neural network or decision tree:
%the tuning of parameters of a simulation or algorithm: 
%it is unknown in advance how these parameters affect the outcome of the machine learning algorithm.
In such applications, we can observe the effect of decisions either in real life, or by running a simulation.
Waiting for such a result can take some time, or may have some other cost associated with it.
Furthermore, the outcome of two observations with the same decision variables may be different.
%In the past decades, this problem has received a lot of attention insimilar fields such as 
Such problems have been approached using methods such as
black-box or Bayesian optimization~\cite{jones1998efficient}, simulation-based optimization~\cite{gosavi2003simulation}, and derivative-free optimization~\cite{conn2009introduction}.
Here, a model fits the relation between decision variables and objective function, and then standard optimization techniques are used on the model instead of the original objective.
These so-called surrogate modeling techniques have been applied successfully to continuous optimization problems in
%all kinds of applications,
signal processing~\cite{DONEpaper}, optics~\cite{DONEpaper}, machine learning~\cite{snoek2012practical}, robotics~\cite{martinez2009bayesian}, and more.
However, it is still an on-going research question on how these techniques can be applied effectively to combinatorial optimization problems.
A common approach is to simply round to the nearest integer, a method that is known to be sub-optimal in traditional optimization, and also in black-box optimization~\cite{garrido2017dealing}.
Another option is to use discrete surrogate models from machine learning such as regression trees \cite{verwer2017auction} or linear model trees \cite{verbeeck2013multi}. Although powerful, this makes both model fitting and optimization computationally expensive.

% Simply rounding solutions to the nearest integer is one possibility, but this approach is known to lead to 
% but how to apply them in combinatorial optimization problems is still an on-going research topic.

This work describes an approach where the surrogate model is still continuous, but where finding the optimum of the surrogate model gives an integer solution.
The main contributions are as follows:
\begin{itemize}
    \item This surrogate modeling algorithm, called IDONE, with two variants (one with a basic and one with a more complex surrogate model).% for combinatorial optimization problems that always returns integer solutions.
    \item A proof that finding the optimum of the surrogate model gives an integer solution.
    \item Experimental results that show when IDONE outperforms random search, simulated annealing and Bayesian optimization.
\end{itemize}

Section~\ref{sec:lit} gives a general description of the problem and an overview of related work.
Section~\ref{sec:alg} describes the IDONE algorithm and the proof.
In Section~\ref{sec:results}, IDONE is compared to random search, simulated annealing and Bayesian optimization on two different problems: finding robust routes in a noise-perturbed traveling salesman benchmark problem, and a convex binary optimization problem.
%on a toy example and also to Bayesian optimization a noisy traveling salesman benchmark problem.
%, as well as a visualization of the surrogate model on a simple traveling salesman problem.
Finally, Section~\ref{sec:conclusion} contains conclusions and future work.

% Motivation

% Novelty?

% Relevance

% What has already been done, what is new.

% Contributions (use bullet points)

\section{Problem description and related work}\label{sec:lit}

Consider the problem of minimizing an objective $f: \R^{d} \rightarrow \R$ with integer and bound constraints:
\begin{align}
     \min_{\x} &\  f(\x) \nonumber\\
    \mathbf{s.t.\ } &\  \x \in \Z^d,\nonumber\\
     &\  \lb_i \leq x_i \leq \ub_i, \ i=1, \ldots, d.
    \label{eq:mainproblem}
\end{align}
%The bound constraints are taken element-wise and the bounds are also assumed to be integer.
These bounds are also assumed to be integer.
It is assumed that $f$ does not have a known mathematical formulation, and can only be accessed via noisy measurements $y=f(x)+\epsilon$, $y\in \R$, with $\epsilon \in \R$ zero-mean noise with finite variance.
Furthermore, taking a measurement $y$ is computationally expensive or is expensive due to a long measuring time, human decision making, deteriorating biological samples, or other reasons.
%The most common examples are those where $y$ is the outcome of an algorithm, a simulation, or a scientific experiment, for example: 
Examples are hyper-parameter optimization in deep learning~\cite{CDONEpaper}, contamination control of a food supply chain~\cite{baptista2018bayesian}, and structure design in material science~\cite{ueno2016combo}.

Although many standard optimization methods are unfit for this problem, there exists a vast number of methods that were designed with most of the above assumptions in mind.
For example, local search heuristics~\cite{aarts2003local} such as hill-climbing, simulated annealing, or taboo search are general enough to be applied to this problem, and have the advantage of being easy to implement.
These heuristics are considered as the baseline in this work, \red{together with random search (simply assign the variables completely random values and keep the best results)}.
Population-based heuristics such as genetic algorithms~\cite{rajeev1992discrete}, particle swarm optimization~\cite{kennedy1997discrete}, and ant colony optimization~\cite{dorigo1999ant} operate in the same way as local search algorithms, but keep track of multiple candidate solutions that together form a population.
%Though they are very successful in many applications, these
These algorithms 
have been applied successfully in many applications, but
are unfit for the problem described in this paper since evaluating a whole population of candidate solutions is not practical if each measurement is expensive.
The same holds for algorithms that filter out the noise by averaging, such as COMPASS~\cite{hong2006discrete}, since they evaluate one candidate solution multiple times.
Typical integer programming algorithms such as branch and bound~\cite{MORRISON201679} work well on standard combinatorial problems, but when the objective function is unknown and observations of single values are noisy or expensive, it is very difficult to obtain reasonable bounds.

Surrogate modeling techniques operate in a different way from the above methods: past measurements are used to fit a model, which  \red{is then} used to select a next candidate solution.
Bayesian optimization algorithms~\cite{jones1998efficient,mockus2012bayesian}, for example, have been successfully applied in many different fields. These methods use probability theory to determine the most promising candidate point according to the surrogate model.
However, when the variables are discrete, the typical approach is to relax the integer constraints, which often leads to sub-optimal solutions~\cite{garrido2017dealing}.
The authors in~\cite{garrido2017dealing}  
%\red{the problem of sub-optimality for discrete problems differently, namely} 
tackled this problem
by modifying the covariance function \red{used in the surrogate model}.
\red{Another} approach\red{,} based on semi-definite programming\red{,} is given in~\cite{baptista2018bayesian}.
\red{And the HyperOpt algorithm~\cite{bergstra2013hyperopt} takes yet a different approach by using a Tree-structured Parzen Estimator as the surrogate model, which is discrete in case the variables are discrete.
HyperOpt is considered the main contender in this paper.
}
%Though these authors successfully applied Bayesian optimization techniques on combinatorial problems, there remains one downside of Bayesian optimization algorithms in general: an ever-increasing computation time.
%However, since 

A downside of many Bayesian optimization algorithms is that the computation time per iteration scales quadratically with the number of measurements taken up until that point.
This causes these methods to become slower over time, and after a certain number of iterations they may even violate the assumption that the  bottleneck in the problem is the cost of evaluating the objective.
%\red{By only keeping track of the latest few measurements and forgetting the earlier ones, the increase in computation time per iteration can be limited, at the cost of accuracy of the surrogate model and stability of the algorithm.
%The HyperOpt algorithm makes use of this technique [NOT TRUE, HYPEROPT FORGETS EARLIER DATA POINTS BUT NOT COMPLETELY, STILL SAVES THEM].
%}
%The growing computation time of Bayesian optimization algorithms
This downside
is recognized and overcome in two similar algorithms: COMBO~\cite{ueno2016combo} and DONE~\cite{DONEpaper}.
Both algorithms use the power of random features~\cite{rahimi2008uniform} to get a fixed computation time every iteration, but COMBO is designed for combinatorial optimization while DONE is designed for continuous optimization.
A disadvantage of COMBO is that it evaluates the surrogate model at every possible candidate point, a number that grows exponentially with the input dimension $d$.
Though evaluating the surrogate model takes very little time (compared to evaluating the original objective $f$), this still makes the algorithm unfit for problems where the input dimension $d$ is large.
A variant of DONE named CDONE has been applied to a mixed-integer problem, where the integer constraints were relaxed~\cite{CDONEpaper}, but as mentioned earlier, this can lead to sub-optimal solutions.

However, the downside of having to relax the integer constraints can be circumvented.
By choosing the basis functions in a certain way, we show how a model can be constructed for which it is known beforehand that the minima of the model lie exactly in integer points.
This makes it possible to apply the algorithm to combinatorial problems, as explained in the next section.

\section{IDONE algorithm}\label{sec:alg}

The DONE algorithm~\cite{DONEpaper} and its variants are able to solve problem~\eqref{eq:mainproblem} without the integer constraint by making use of a surrogate model.
Every time a new measurement $y=f(\x)+\epsilon$ comes in, the surrogate model is updated, the minimum of the surrogate model is found, and an exploration step is performed.
%If this approach is used on combinatorial optimization problems, rounding has to be performed either after finding the minimum of the surrogate model, or in the function $f$ itself.
%Both approaches may lead to sub-optimal solutions.
To make the algorithm suitable for combinatorial optimization
we propose a variant of DONE called IDONE (integer-DONE), where the surrogate model is guaranteed to have integer-valued minima.
%, making it more suitable for combinatorial optimization.

% First, the proposed surrogate model and its local minima are investigated in Section~\ref{sec:surrogatemodel}.
% Then, in Section~\ref{sec:training} it is described how the model is updated, followed by an explanation in Section~\ref{sec:minimize} of how the minimum of the model is found.
% After this, the exploration step is described in Sections~\ref{sec:exploration}.
% Finally, the pseudo-code for the full algorithm is given in  Section~\ref{sec:pseudocode}.

\subsection{Piece-wise linear surrogate model}\label{sec:surrogatemodel}

The proposed surrogate model $g: \R^d \rightarrow \R$ is a linear combination of rectified linear units $\mathrm{ReLU}(z) = \max(0,z)$, a basis function that is commonly used in the deep learning community~\cite{lecun2015deep}:
\begin{align}
    g(\x) & = \sum_{k=1}^D c_k \max\left\{0, z_k(\x)\right\},\nonumber\\
    z_k(\x) & = \w_k^T \x + b_k, \label{eq:model}
\end{align}
with $\x\in \R^d$ and $\w_k\in \R^d, b_k\in \R, c_k\in \R$ for $k=1,\ldots, D$.
%This can also be seen as a deep neural network with only one hidden layer, also called a multilayer perceptron.
%These networks are commonly used to approximate nonlinear functions.
Unlike what is common practice in the deep learning community, the parameters $\w_k$ and $b_k$ remain fixed in this surrogate model.
This makes the model linear in its parameters ($c_k$), allowing it to be trained via linear regression instead of iterative methods.
This is explained in Section~\ref{sec:training}.

Because of the choice of basis functions, the surrogate model is actually piece-wise linear, which causes its local minima to lie in one of its corner points:
%This makes it possible to calculate the possible locations for the local minima of the model beforehand:
%This gives simple conditions on the local minima of the model:
%
\begin{theorem}\label{thm:localminima}
Any strict local minimum of $g$ lies in a point $\x^*$ with $z_k(\x^*) = 0$ for $d$ linearly independent $z_k$.
\end{theorem}

%This is due to the well-known fact that a strict local minimum of a piece-wise linear model must lie in one of its corner points.
%For completeness, a proof is given in Appendix~\ref{sec:proof}.

%The proof is given in the Appendix.
The reverse of this theorem is not necessarily true: if $\hat \x$  satisfies $z_k(\hat \x) = 0$ for $d$ linearly independent $z_k$, then it  depends on the parameters $c_k$ of the model whether $\hat \x$ is actually a local minimum or not.

%In the CDONE algorithm~\cite{CDONEpaper}, the parameters $\w_k$ and $b_k$ were chosen randomly from continuous probability distributions, leading to any combination of $d$ different $z_k$ being linearly independent with probability $1$.
%For a high number of basis functions $D>>d$, this allows the model to have its local minima in many different possible locations. 
%This flexibility is good for continuous optimization, but for discrete optimization it is better that finding the minimum of the model always leads to an integer solution.

The number of local minima and their locations depend on the parameters $\w_k$ and $b_k$.
In this work, we provide two options for choosing these parameters in such a way that the local minima are always found in integer solutions.
%We provide two possible choices for $\w_k$, $b_k$.
In the first case, the functions $z_k$ are simply chosen to have zeros on hyper-planes that together form an integer lattice:
\begin{definition}[\red{Basic} model ]
 %Let the number of basis functions \\ \mbox{$D=2||\lb-\ub||_1+1$}.
 Let $g$ be as in~\eqref{eq:model}.
 The parameters $\w_k$, $b_k$ of \red{the basic model} are chosen according to Algorithm~\ref{alg:model1}.
%
%
%  \begin{align*}
%      \w_1 & =[0, \ldots, 0]^T, \ b_1=1 \ (\mathrm{model \  bias}),\nonumber\\
%      \w_{2} & =[1, 0, \ldots, 0]^T, \ b_2=-[\lb]_1,\nonumber\\
%      \w_3 & = [-1, 0, \ldots, 0]^T, \ b_3=[\lb]_1+1, \nonumber\\
%      \w_4 & = [1, 0, \ldots, 0]^T, \ b_4=-([\lb]_1+1), \nonumber\\
%      \w_5 & = [-1, 0, \ldots, 0]^T, \ b_5=[\lb]_1+2, \nonumber\\
%      \w_6 & = [1, 0, \ldots, 0]^T, \ b_6=-([\lb]_1+2), \nonumber\\
%      & \ \  \vdots \nonumber\\
%      \w_{2([\ub]_1-[\lb]_1+1)+1} & = [-1, 0, \ldots, 0]^T, \  b_{2([\ub]_1-[\lb]_1)+1}=[\ub]_1-1, \nonumber\\
%      \w_{2([\ub]_1-[\lb]_1)+2} & = [1, 0, \ldots, 0]^T, \  b_{2([\ub]_1-[\lb]_1)+2}=-([\ub]_1-1), \nonumber\\
%      \w_{2([\ub]_1-[\lb]_1)+3} & = [-1, 0, \ldots, 0]^T, \  b_{2([\ub]_1-[\lb]_1)+3}=[\ub]_1, \nonumber\\
%       \w_{2([\ub]_1-[\lb]_1)+4} & = [0, 1, 0, \ldots, 0]^T, \  b_{2([\ub]_1-[\lb]_1)+4}=-[\lb]_2, \nonumber\\
%       &\ \  \vdots\\
%       \w_{2||\ub-\lb||_1+1} & = [0,\ldots, 0, -1]^T, \  b_{2||\ub-\lb||_1+1}=[\ub]_d.
%  \end{align*}
 That is, every function $z_k$ is zero on a $(d-1)$-dimensional hyper-plane with $x_i=j$ for some dimension $i\in\{1,\ldots, d\}$ and some integer $\lb_i\leq j \leq \ub_i$.
\end{definition}

% \begin{algorithm}[htbp]
% \caption{Model 1 parameters}\label{alg:model1}
%  \begin{algorithmic}
% \State $k \leftarrow 1$
% \State $\w_k\leftarrow[0, \ldots, 0]^T, b_k\leftarrow1$
% \Comment{model bias}
% \State $k \leftarrow k+1$
% \For{$i=1, \ldots, d$}
%     \State $\w_k\leftarrow[0, \ldots, 0]^T$, $\w_{k,i} \leftarrow 1$, $b_k \leftarrow -\lb_i$
%     \State $k \leftarrow k+1$
%     \For{$j=\lb_i+1, \ldots, \ub_i-1$}
%         \State $\w_k\leftarrow[0, \ldots, 0]^T$, $\w_{k,i} \leftarrow 1$,  $b_k \leftarrow -j$
%         \State $k \leftarrow k+1$
%         \State $\w_k\leftarrow[0, \ldots, 0]^T$, $\w_{k,i} \leftarrow -1$,  $b_k \leftarrow j$
%         \State $k \leftarrow k+1$
%     \EndFor
%     \State $\w_k\leftarrow[0, \ldots, 0]^T$, $\w_{k,i} \leftarrow -1$, $b_k \leftarrow \ub_i$
%     \State $k \leftarrow k+1$
% \EndFor
% \end{algorithmic}
% \end{algorithm}
\begin{algorithm}[tbp]
\caption{\red{Basic} model parameters}\label{alg:model1}
 \begin{algorithmic}
 \Require $d$, $l_i$, $u_i$, $i=1,\ldots, d$
 \Ensure $\w_k$, $b_k$, $k=1,\ldots, D$
\State $k \leftarrow 1$
\State $\w_k\leftarrow[0, \ldots, 0]^T, b_k\leftarrow1$, $k \leftarrow k+1$
\Comment{model bias}
\For{$i=1, \ldots, d$}
    \For{$j=\lb_i, \ldots, \ub_i$}
        \State $\w \leftarrow \mathbf{e}_i$ \Comment{$\mathbf{e}_i = $ Unit vector in dimension $i$}
        \State $b=-j$
        \If{$j=\lb_i$} \Comment{Lower bound}
            \State $\w_k\leftarrow \w$, $b_k\leftarrow b$, $k\leftarrow k+1$
        \ElsIf{$j=\ub_i$} \Comment{Upper bound}
            \State $\w_k\leftarrow -\w$, $b_k\leftarrow -b$, $k\leftarrow k+1$
        \Else \Comment{Between the bounds}
            \State $\w_k\leftarrow \w$, $b_k\leftarrow b$, $k\leftarrow k+1$
            \State $\w_k\leftarrow -\w$, $b_k\leftarrow -b$, $k\leftarrow k+1$
        \EndIf
    \EndFor
\EndFor
\end{algorithmic}
\end{algorithm}

An example of a basis function in this model is $\max\{0,z_k(\x)\}$ with $z_k(\x)=x_5-3$.
This model has $D=1+2\sum_{i=1}^d{\ub_i-\lb_i}$ basis functions in total.
The $1$ comes from the model bias, a basis function that is equal to $1$ everywhere.
This allows the model to be shifted up or down.

Since all the basis functions depend only on one variable, this \red{basic} model might not be fit for problems where the decision variables have complex interactions.
Therefore, in the \red{advanced} model, we use the same basis functions, but we also add basis functions that depend on two variables:
\begin{definition}[\red{Advanced} model]
 Let $g$ be as in~\eqref{eq:model}.
 The parameters $\w_k$ and $b_k$ of \red{the advanced} model are chosen according to Algorithm~\ref{alg:model2}.
%  \begin{align*}
%     D_1& =2||\ub-\lb||_1+1,\\
%     \w_{D_1+1} & =[-1, 1, 0, \ldots, 0]^T, \ b_{D_1+1}=-([\lb]_2-[\ub]_1),\nonumber\\
%     \w_{D_1+2} & =[1, -1, 0, \ldots, 0]^T, \ b_{D_1+2}=[\lb]_2-[\ub]_1+1,\nonumber\\
%     \w_{D_1+3} & =[-1, 1, 0, \ldots, 0]^T, \ b_{D_1+3}=-([\lb]_2-[\ub]_1+1),\nonumber\\
%     \w_{D_1+4} & =[1, -1, 0, \ldots, 0]^T, \ b_{D_1+4}=[\lb]_2-[\ub]_1+2,\nonumber\\
%     \w_{D_1+5} & =[-1, 1, 0, \ldots, 0]^T, \ b_{D_1+5}=-([\lb]_2-[\ub]_1+2),\nonumber\\
%     & \ \ \vdots\\
%     D_2 & = D_1 + ([\ub]_2-[\lb]_2)+([\ub]_1-[\lb]_1), \\
%     \w_{D_2-2} & =[1, -1, 0, \ldots, 0]^T, \ b_{D_2-2}=[\ub]_2-[\lb]_1-1,\nonumber\\
%     \w_{D_2-1} & =[-1, 1, 0, \ldots, 0]^T, \ b_{D_2-1}=-([\ub]_2-[\lb]_1-1),\nonumber\\
%     \w_{D_2} & =[1, -1, 0, \ldots, 0]^T, \ b_{D_2}=[\ub]_2-[\lb]_1,\nonumber\\
%     \w_{D_2+1} & =[0, -1, 1, 0, \ldots, 0]^T, \ b_{D_2+1}=-([\lb]_3-[\ub]_2),\nonumber\\
%     & \ \ \vdots\\
%     D_{D+1} & = D_1+\sum_{l=2}^D ([\ub]_l-[\lb]_l)+([\ub]_{l-1}-[\lb]_{l-1})]),\\
%     \w_{D_{D+1}} & = [0, \ldots, 0, 1, -1]^T, \ b_{D_{D+1}}=[\ub]_d-[\lb]_{d-1}.
%  \end{align*}
 That is, every function $z_k$ different from the ones in \red{the basic model} is zero on a $(d - 1)$-dimensional hyper-plane with $x_i-x_{i-1}=j$ for some dimension $i\in\{2, \ldots, d\}$ and some integer $\lb_i-\ub_{i-1} \leq j \leq \ub_i - \lb_{i-1}$.
\end{definition}

% \begin{algorithm}[htbp]
% \caption{Model 2 parameters}\label{alg:model2}
%  \begin{algorithmic}
%  \State Perform Algorithm~\ref{alg:model1}.
% \For{$i=2, \ldots, d$}
%     \State $\w_k\leftarrow[0, \ldots, 0]^T$, $\w_{k,i} \leftarrow 1$, $\w_{k,i-1} \leftarrow -1$, $b_k \leftarrow -(\lb_i-\ub_{i-1})$
%     \State $k \leftarrow k+1$
%     \For{$j=\lb_i-\ub_{i-1}+1, \ldots, \ub_i-\lb_{i-1}-1$}
%         \State $\w_k\leftarrow[0, \ldots, 0]^T$, $\w_{k,i} \leftarrow 1$, $\w_{k,i-1} \leftarrow -1$,  $b_k \leftarrow -j$
%         \State $k \leftarrow k+1$
%         \State $\w_k\leftarrow[0, \ldots, 0]^T$, $\w_{k,i} \leftarrow -1$,  $\w_{k,i-1} \leftarrow 1$, $b_k \leftarrow j$
%         \State $k \leftarrow k+1$
%     \EndFor
%     \State $\w_k\leftarrow[0, \ldots, 0]^T$, $\w_{k,i} \leftarrow -1$, $\w_{k,i-1} \leftarrow 1$, $b_k \leftarrow \ub_i-\lb_{i-1}$
%     \State $k \leftarrow k+1$
% \EndFor
% \end{algorithmic}
% \end{algorithm}

\begin{algorithm}[htbp]
\caption{\red{Advanced model} parameters}\label{alg:model2}
 \begin{algorithmic}
 \Require $d$, $l_i$, $u_i$, $i=1,\ldots, d$
 \Ensure $\w_k$, $b_k$, $k=1,\ldots, D$
 \State Perform Algorithm~\ref{alg:model1}.
\For{$i=2, \ldots, d$}
    \For{$j=\lb_i-\ub_{i-1}, \ldots, \ub_i-\lb_{i-1}$}
        \State $\w \leftarrow \mathbf{e}_i-\mathbf{e}_{i-1}$ \Comment{$\mathbf{e}_i = $ Unit vector in dimension $i$} %\Comment{$w_i=1$, $w_{i-1}=-1$} 
        \State $b=-j$
        \If{$j=\lb_i-\ub_{i-1}$} \Comment{Lower bound}
            \State $\w_k\leftarrow \w$, $b_k\leftarrow b$, $k\leftarrow k+1$
        \ElsIf{$j=\ub_i-\lb_{i-1}$} \Comment{Upper bound}
            \State $\w_k\leftarrow -\w$, $b_k\leftarrow -b$, $k\leftarrow k+1$
        \Else \Comment{Between the bounds}
            \State $\w_k\leftarrow \w$, $b_k\leftarrow b$, $k\leftarrow k+1$
            \State $\w_k\leftarrow -\w$, $b_k\leftarrow -b$, $k\leftarrow k+1$
        \EndIf
    \EndFor
\EndFor
\end{algorithmic}
\end{algorithm}

This model has $D = 2\sum_{i=2}^D \ub_i-\lb_i+\ub_{i-1}-\lb_{i-1}$ more basis functions than \red{the basic model}.
The added functions $z_k$ are zero on diagonal lines through two variables, see Figure~\ref{fig:basisfunc}.
%This choice of basis functions seems natural for problems where there is an ordering between the decision variables, though it can still be used if this is not the case.

%eps_figures
\begin{figure}[tb]
\centering
\includegraphics[width=0.45\columnwidth]{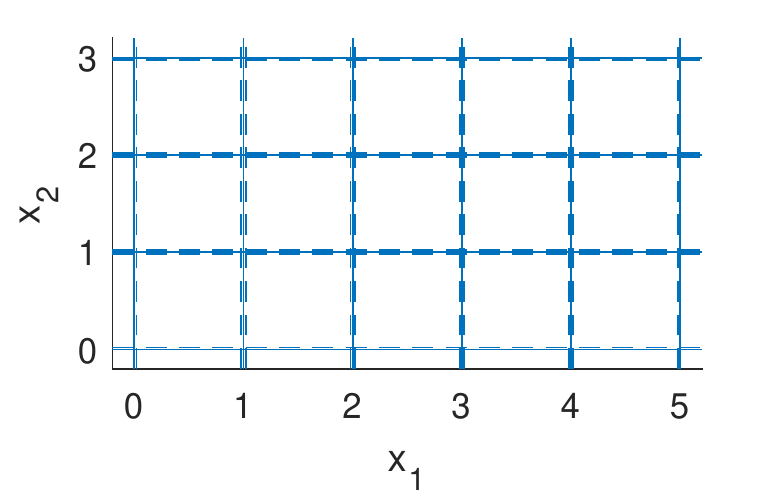}
\includegraphics[width=0.45\columnwidth]{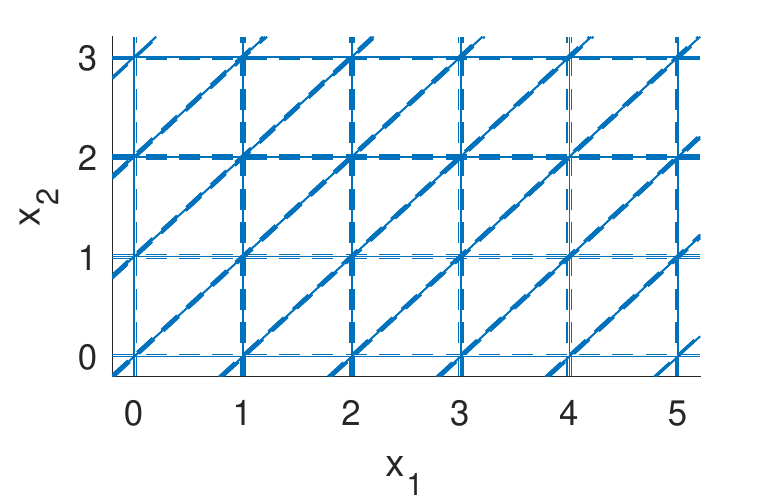}
\caption{Regions where the functions $z_k$ of the surrogate model are exactly zero, for \red{the basic model} (left) and \red{the advanced model} (right), for a problem with two variables with lower bounds $(0,0)$ and upper bounds $(5,3)$. The functions $z_k$ have been chosen in such a way that they cross exactly at integer points within the bounded region. This ensures that the model has its minimum in one of these points, making the model more suitable for combinatorial optimization problems.} \label{fig:basisfunc}
\end{figure}

We now show one of our main contributions.
\begin{theorem}\label{thm:main}
    (I) %Any strict local minimum of Model 1 lies in a point $\x^*$ that satisfies all constraints of the original problem~\eqref{eq:mainproblem}.
    If $\x^*$ is a strict local minimum of \red{the basic model}, then $\x^*\in \Z^d$ and $\lb_i \leq x_i \leq \ub_i, \ \forall i=1, \ldots, d$.
    
    \noindent (II) %Any strict local minimum of Model 2 that lies in a point $\x^*$ that satisfies the bound constraints in~\eqref{eq:mainproblem} satisfies $\x^*\in \Z^d$.
    If $\x^*$ is a strict local minimum of \red{the advanced model},
    %and satisfies $\lb_i \leq x_i \leq \ub_k, \ \forall i=1, \ldots, d$
    then $\x^*\in \Z^d$.
    
    \noindent (III) If $\x^*$ is a non-strict local minimum of \red{the basic model}, it holds that the model retains the same value when going from $\x^*$ to the nearest point $\hat \x$ that satisfies $\hat \x \in \Z^d$ and $\lb_i \leq \hat x_i \leq \ub_k, \ \forall i=1, \ldots, d$.
    %all constraints of~\eqref{eq:mainproblem}.
    
    \noindent (IV) If $\x^*$ is a non-strict local minimum of \red{the advanced model},
    %that also satisfies    $\lb_i \leq x_i \leq \ub_k, \ \forall i=1, \ldots, d$,
    %the bound constraints in~\eqref{eq:mainproblem}, 
    it holds that the model retains the same value when rounding $\x^*$ to the nearest integer.
\end{theorem}
\begin{proof}
(I) Let $\x^*$ be a strict local minimum of \red{the basic model}.
By Theorem~\ref{thm:localminima}, there are $d$ linearly independent $z_k$ with $z_k(\x^*)=0$.
From Algorithm~\ref{alg:model1} it can be seen that all functions $z_k$ have the form $z_k(\x) = \pm(x_i-j)$, for some $i=1,\ldots, d$, $j=\lb_i,\ldots, \ub_i$.
Since $d$ of these functions are linearly independent, all $d$ of them must have a different $i$.
Since all $d$ of them satisfy $z_k(\x^*)=0$, it holds that $x^*_i=j$, for some $j=\lb_i,\ldots,\ub_i$, $\forall i=1,\ldots, d$, which is what is claimed.

% Using Theorem~\ref{thm:localminima}, the locations of the strict local minima for both Model 1 and 2 can be calculated by looking at the null-space of $d$ linearly independent functions $z_k$.
% For Model 1, all the functions $z_k$ have the form $z_k(\x) = \pm(x_i-j)$, for some $i=1,\ldots, d$, $\lb_i\leq j \leq \ub_i$.
% When $d$ of these functions are linearly independent, all $d$ of them must have a different $i$.
% %It can be seen that for two different $k_1$, $k_2$, one of the following cases holds: A) $z_{k_1}$ is parallel to $z_{k_2}$, B) $z_{k_1}=-z_{k_2}$, C) $z_{k_1}$ is orthogonal to $z_{k_2}$. 
% %They can only be linearly independent in case C), and this can happen for up to $d$ different functions.
% When these $d$ different functions are zero, %looking at the parameters $\w_k$ and $b_k$
% it holds that $x_i=j$, for $i=1,\ldots, d$, $\lb_i\leq j \leq \ub_i$. This proves part (I).

 (II) Let $\x^*$ be a strict local minimum of \red{the advanced model}.
 By Theorem~\ref{thm:localminima}, there are $d$ linearly independent $z_k$ with $z_k(\x^*)=0$.
 This means that all $z_k$ together must depend on all $x_i$, $i=1,\ldots, d$.
 From Algorithm~\ref{alg:model2} it can be seen that all functions $z_k$ have the same form as in \red{the basic model}, that is, $z_k(\x) = \pm(x_i-j)$, $i=1,\ldots, d$, $j=\lb_i,\ldots, \ub_i$, or they have the form $z_k(\x) = \pm(x_i-x_{i-1} - j)$, $i=2,\ldots, d$, $j=\lb_i-\ub_{i-1}, \ldots, \ub_i - \lb_{i-1}$.
 No matter the form, thus
 \begin{align}\label{eq:difx}
     x^*_i-x^*_{i-1}\in \Z\ \forall i=2,\ldots, d.
 \end{align}
 To arrive at a contradiction, suppose that $\exists s\in\{1,\ldots, d\}$ such that $x^*_s \not\in \Z$.
 Then by~\eqref{eq:difx}, $x^*_i\not\in\Z$ for all $i\in\{1,\ldots,d\}$.
 However, this is only possible if none of the $z_k$ have the form $z_k(\x) = \pm(x_i-j)$, and all $d$ of the $z_k$ have the form $z_k(\x) = \pm(x_i-x_{i-1} - j)$, for $d$ different $i$.
 But by construction, there are only $d-1$ of these last ones available, see Algorithm~\ref{alg:model2} (the for-loop starts at $2$).
 Therefore, it is not true that $\exists s\in\{1,\ldots, d\}$ such that $x^*_s \not\in \Z$.
 Hence, $x^*_i\in \Z \ \forall i=1,\ldots, d$, which is what the theorem claims.

(III) Let $\x^*$ be a non-strict local minimum of \red{the basic model} and suppose $x^*_s \not\in \Z \cup [\lb_s, \ub_s]$ for some $s\in\{1,\ldots, d\}$.
Let $L$ be the line segment from $x^*_s$ to the nearest point $\hat x_s$ in the set $\Z \cup [\lb_s, \ub_s]$, without including that point.
Since the only $z_k$ functions that depend on $x_s$ have the form 
$z_k(\x) = \pm(x_s-j)$, $j=\lb_s,\ldots, \ub_s$, it follows that $z_k(\x)=0$ does not happen on $L$ for any $z_k$ that depends on $x_s$.
Therefore, model $g$ is linear on this line segment, and since $\x^*$ is a non-strict local minimum and $g$ is continuous, $g$ retains the same value when replacing $x^*_s$ by $\hat x_s$.
This can be repeated for all $s$ for which $x^*_s \not\in \Z \cup [\lb_s, \ub_s]$, which proves the claim.

(IV) Let $\x^*$ be a non-strict local minimum of \red{the advanced model} and suppose $\x^* \not\in \Z$.
We first show that rounding $\x^*$ to the nearest integer does not change the sign of any $z_k$.
Note that all the $z_k$ of \red{the advanced model} have the form $z_k(\x) = \pm(x_i-j)$ or $z_k(\x) = \pm(x_i-x_{i-1} - j)$, for some $i=1,\ldots, d$ and some integer $j$.
Let $\bar x_i$ denote rounding $x_i$ to the nearest integer.
Then we have (because $j$ is integer): 
\begin{align*}
    x_i\leq j & \Rightarrow \bar x_i \leq j,\text{ and }
    x_i \geq j \Rightarrow \bar x_i \geq j,\\
    x_i-x_{i-1}\leq j &  \Rightarrow \bar x_i \leq \overline{x_{i-1}+j} %= \bar x_{i-1} + j
    \Rightarrow     \bar x_i - \bar x_{i-1} \leq j,\\
    x_i-x_{i-1}\geq j & \Rightarrow \bar x_i - \bar x_{i-1} \geq  j.
\end{align*}
Since the sign of none of the $z_k$ change when rounding, and model $g$ is only nonlinear when going from $z_k(\x) < 0$ to $z_k(\x)>0$ for some $k=1,\ldots, D$, it follows that $g$ is linear on the line segment from $\x^*$ to the nearest integer.
Together with the fact that $\x^*$ is a non-strict local minimum, it follows that $g$ retains the same value on this line segment. 
Finally, the claim is valid because $g$ is continuous.
%\qed
\end{proof}

\subsection{Fitting the model}\label{sec:training}

%As mentioned earlier, for fixed $\w_k$, $b_k$, the surrogate model $g$ is linear in the parameters $c_k$.
%This makes it possible to use linear regression to fit the model. %, rather than iterative techniques.
Because the surrogate model $g$ is linear in its parameters $c_k$, fitting the model can be done with linear regression.
%After evaluating the original objective $f$ several times,
Given input-output pairs $(\x_i, y_i, i=1,\ldots, N$), this is done by solving the regularized linear least squares problem
\begin{align}
    \min_{\c_N} \sum_{n=1}^N \left(y_n - g(\x_n, \c_N)\right)^2 + \lambda ||\c_N-\c_0||^2_2, \label{eq:LS}
\end{align}
with regularization parameter $\lambda$ and initial weights $\c_0$.
The regularization part is added to overcome ill-conditioning, noise, and model over-fitting.
Furthermore, by choosing $\c_0=[0, 1, \ldots, 1]^T$, it is ensured that the surrogate model is convex before the first iteration~\cite{CDONEpaper}.
In this work, $\lambda=0.001$ has been chosen.

% Although the least squares problem~\eqref{eq:LS} has a direct solution for $\c_N$ available, computing this solution has computational complexity $O(N^3)$, with $N$ the number of measurements taken so far.
% Since the least squares problem needs to be solved every iteration of the algorithm, the computational complexity quickly grows to  unacceptable proportions.
% An effective way to lower the complexity of calculating $\c_N$ is by making use of the previous solution $\c_{N-1}$.

To prevent having to solve this problem at every iteration (with runtime $O(N^3)$), \eqref{eq:LS} is solved with
%This is done with 
the recursive least squares algorithm~\cite{sayed1998recursive}.
% :
% \begin{align}
%     P_0 & =\frac 1 \lambda I_{D\times D},\nonumber\\
%     z_k(\x_n) &= \w_k^T \x_n + b_k, \ k=1,\ldots, D,\nonumber\\
%     u(\x_n) & = \left[\max\{0,z_1(\x_n)\}, \ \ldots, \ \max\{0,z_D(\x_n)\}\right]^T,\nonumber\\
%     \alpha_n & = \frac{P_{n-1}u(\x_n)}{1+u(\x_n)^T P_{n-1} u(\x_n)},\nonumber\\
%     P_n & = P_{n-1} - \alpha_n u(\x_n)^T P_{n-1}, \nonumber\\
%     \c_{n} & = \c_{n-1} + \alpha_n \left( y_n - u(\x_n)^T \c_{n-1}\right), \ n=1, \ldots, N \label{eq:RLS}
% \end{align}
% with $I$ the identity matrix.
%giving a recursive formulation for $\c_N$.
This algorithm has runtime $O(D^2)$ per iteration, with $D$ the number of basis functions used by the model.
This implies that the computation time per iteration does not depend on the number of measurements, which is a big advantage compared to Bayesian optimization algorithms (which usually have complexity $O(N^2)$ per iteration).
The memory is also $O(D^2)$, because 
%matrix $P_n$
a $D \times D$ covariance matrix needs to be stored.
Since $D$ scales linearly with the input dimension $d$ and with the lower and upper bounds, the computational complexity of fitting the surrogate model is $O(p^2 d^2)$, with $p=\max_i(\ub_i-\lb_i)$.

%\subsection{Traveling salesman problem (4 cities)}\label{sec:TSP4}
\red{\subsubsection{Model visualization}
\label{sec:TSP4}}

\red{To visualize the surrogate model used by the IDONE algorithm, the fitting procedure is applied to a simple traveling salesman problem with four cities.
%without noise.
The distance matrix for the cities is shown in Table~\ref{tab:simpleTSP}.
The decision variables are chosen as follows: the route starts at city 1, then variable $x_1\in\{1,2,3\}$ determines which of the three remaining cities is visited, then variable $x_2\in\{1,2\}$ determines which of the two remaining cities is visited; then the one remaining city is visited, then city 1 is visited again.
This problem has two optimal solutions: $\x=[1,2]^T$ (route 1-2-4-3-1) and $\x=[2,2]^T$ (route 1-3-4-2-1), both with a total distance of $80$.
All other solutions have a total distance of $95$.
%This problem is too simple to showcase any comparison with other algorithms, but with only two decision variables it is possible to visualize the surrogate model.
%After running IDONE for $100$ iterations, the model is fit through all the measurement points taken so far (see Figure~\ref{fig:simpleTSP}).
}

\red{Figure~\ref{fig:simpleTSP} shows what the surrogate model looks like after taking measurements in all possible data points for this problem, which is possible due to the low number of possibilities.
It can be observed that this model is piece-wise linear and that any local minimum retains the same value when rounding to the nearest integer.
Furthermore, the diagonal lines (see also Figure~\ref{fig:basisfunc}) make the advanced model more accurate.}
%{\color{red} WHAT ELSE SHOULD WE SAY ABOUT THIS FIGURE?}

\begin{table}[tb]
\caption{Distance matrix for the simple \red{traveling salesman problem}}\label{tab:simpleTSP}
\centering
\begin{tabular}{c||c|c|c|c|}
%\hline
%City \textbackslash  Distances & & & \\
& 1 & 2 & 3 & 4\\
\hline
\hline
1 & & $10$ & $15$ & $20$\\
\hline
2 & $10$ & & $35$ & $25$\\
\hline
3 & $15$ & $35$ & & $30$\\
\hline
4 & $20$ & $25$ & $30$ & \\
\hline
\end{tabular}
\end{table}

%eps_figures
\begin{figure}[tb]
\centering
\includegraphics[width=0.45\textwidth]{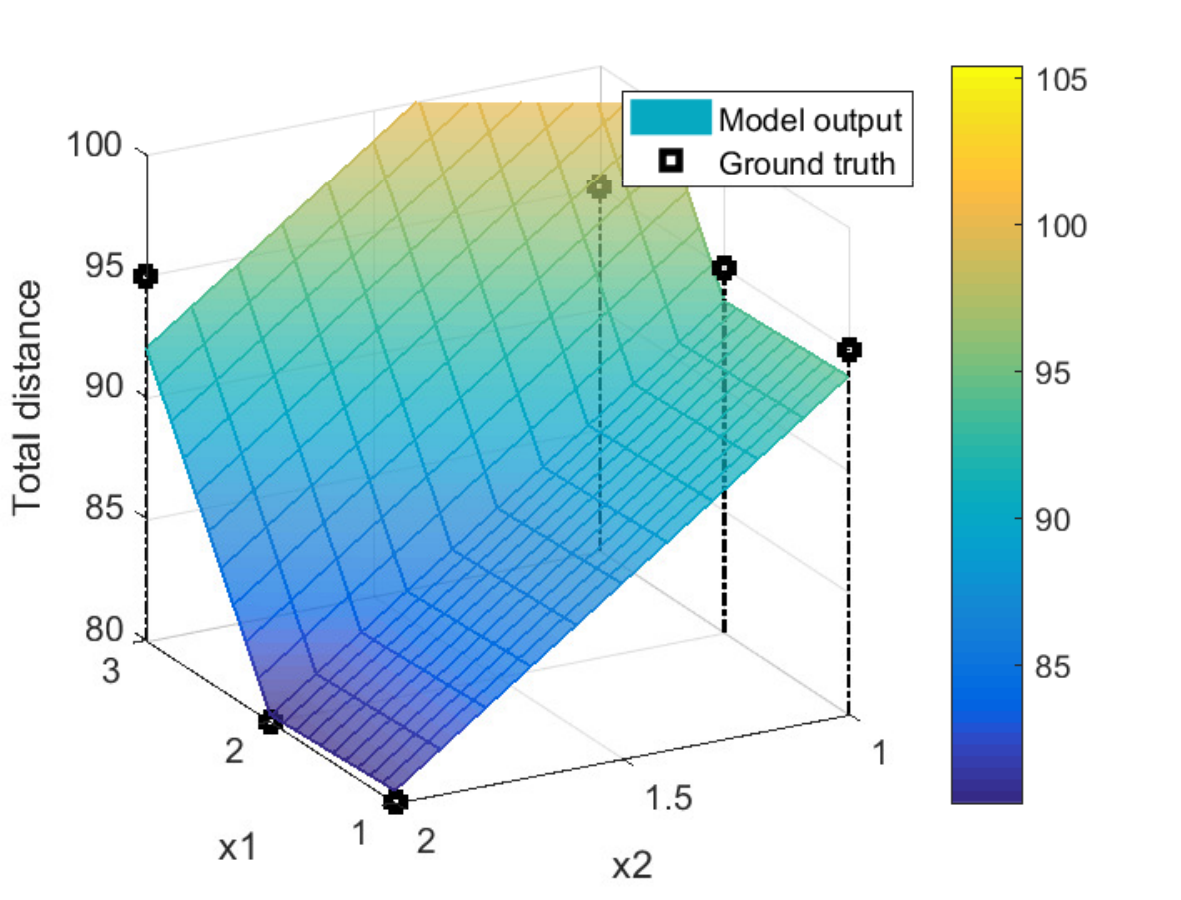}%
\includegraphics[width=0.45\textwidth]{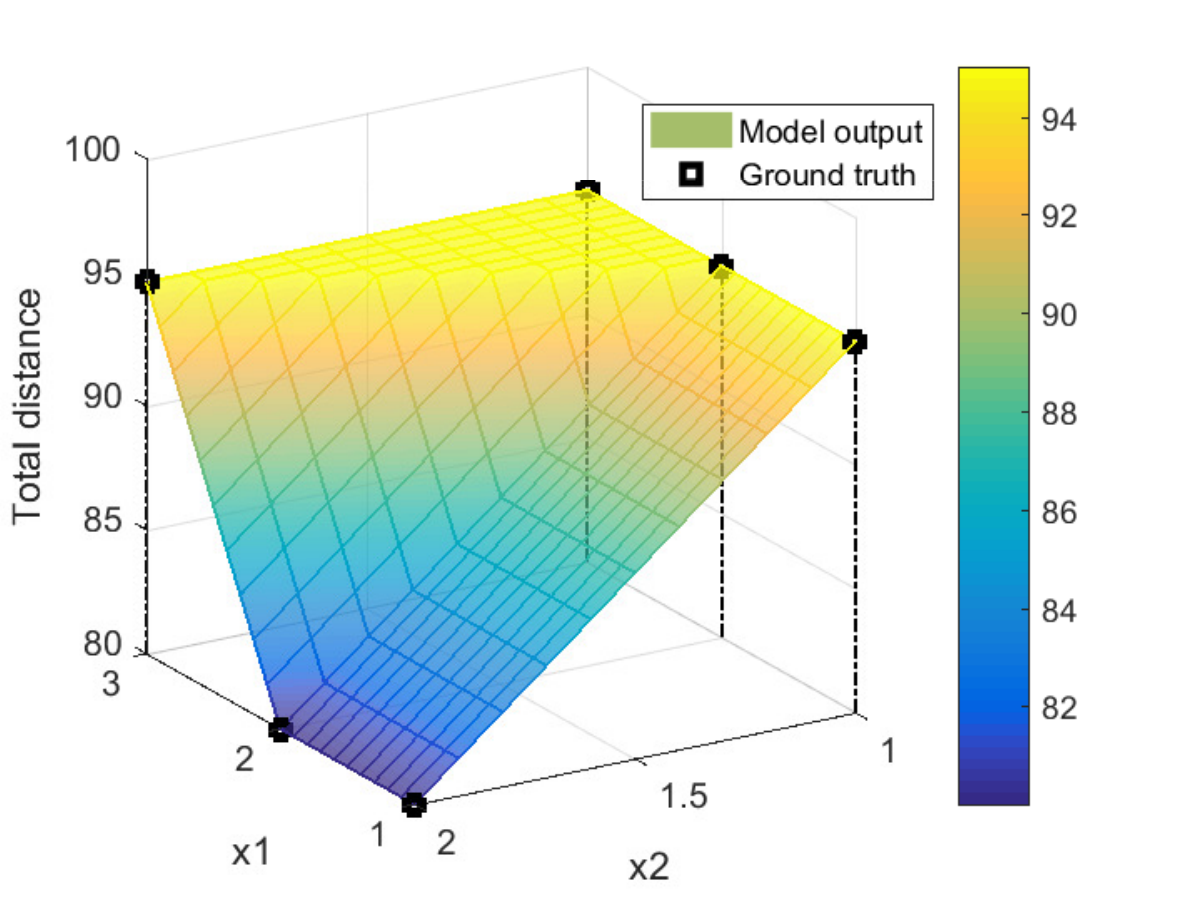}
\caption{Model output for the simple \red{traveling salesman problem} for \red{the basic model} (left) and \red{the advanced model} (right)
%, after $100$ iterations
. The starting city is city 1, $x_1$ determines which remaining city is visited next, $x_2$ determines which remaining city is visited third, then the only remaining city is visited, and then city 1 is visited again.} \label{fig:simpleTSP}
\end{figure}

\subsection{Finding the minimum of the model}\label{sec:minimize}

After fitting the model $g$ at iteration $N$, the algorithm proceeds to find a local minimum using the new weights $\c_N$:
\begin{align}
    \x^* = & \  \arg\min_{\x} g(\x,\c_N),\nonumber\\
         \ \mathbf{s.t.\ } & \  \x \in \Z^d,\nonumber\\
     & \ \lb_i  \leq x_i \leq \ub_k, \ i=1, \ldots, d.
     \label{eq:minmodel}
\end{align}
%Since model $g$ does have a mathematical formulation, does not contain noise, and is very fast to evaluate, problem~\eqref{eq:minmodel} can be solved with standard solvers.
%The Matlab code \emph{fmincon} 
\red{The BFGS method~\cite{nocedal}}
with a relaxation on the integer constraint 
was used to solve the above problem, with a provided analytical derivative of $g$. 
\red{In this work, }the derivative of the basis function $\mathrm{ReLU}(z)=\max(0,z)$ has been chosen to be $0.5$ at $z=0$.
\red{The optimal solution} was rounded to the nearest integer per Theorem~\ref{thm:main}.

%The solver was stopped after $20$ iterations (these are internal iterations of \emph{fmincon}, not of the IDONE algorithm) and the solution was rounded to the nearest integer (per Theorem~\ref{thm:main}).
% Because of the results of Theorem~\ref{thm:main}, relaxing the integer constraint is justified, though there are still many possibilities to improve this part of the algorithm.
% For example, in the CDONE algorithm~\cite{CDONEpaper}, the sparsity of the surrogate model has been investigated.
% Exploiting the sparsity, piece-wise linearity or other special structures of the surrogate model could lead to much more efficient implementations.
% This remains for future work.

\subsection{Exploration}\label{sec:exploration}

% [We don't focus on this part either: just add a small integer to all the variables with a small probability, just like with simulated annealing]

% [Other possibilities: use an acquisition function like in Bayesian optimization, or look at the other methods in the literature to see what they use]

After fitting the model and finding its minimum, a new point $\x_{N+1}$ needs to be chosen to evaluate the function $f$.
As in DONE~\cite{DONEpaper}, 
% and CDONE~\cite{CDONEpaper} algorithms, 
a random perturbation $\delta$ is added to the found minimum:
$    \x_{N+1}  = \x^* + \delta$,
but instead of a continuous random variable, $\delta\in\{-1,0,1\}^d$ is a discrete random variable with the following probabilities:
\begin{align}\label{eq:explo2}
    P(\delta_i=0) & = 1-p,\nonumber\\
    P(\delta_i=1) &= \left\{\begin{array}{ll}p, & x^*_i=\lb_i,\\ 0, &  x^*_i=\ub_i,\\ p/2, & \mathrm{otherwise},\end{array}\right. \nonumber\\
    P(\delta_i=-1) &= \left\{\begin{array}{ll}0, & x^*_i=\lb_i,\\ p, &  x^*_i=\ub_i,\\ p/2, & \mathrm{otherwise}.\end{array}\right. 
\end{align}

In this work, $p=1/d$ has been chosen \red{($d$ is the number of variables)}.
% Without this exploration step, the algorithm might get stuck if the minimum of the surrogate model does not change after the model is updated.
% Furthermore, this step helps to avoid local minima in the original objective function.
%The exploration step is skipped at the final iteration.

\subsection{IDONE algorithm}\label{sec:pseudocode}

The IDONE algorithm iterates over three phases: updating the surrogate model with recursive least squares, finding the minimum of the model, and performing the exploration step.
The pseudocode for the algorithm is shown in Algorithm~\ref{alg:IDONE}.
Depending on which subroutine is used in the first line, we refer to this algorithm as either \red{IDONE-basic (using the basic model) or IDONE-advanced (using the advanced model).
}

\begin{algorithm}[htbp]
\caption{\red{IDONE-advanced, IDONE-basic}}\label{alg:IDONE}
 \begin{algorithmic}
 \Require $\x_1 \in \R^d$, $\lambda \in \R$, $(\lb_i, \ub_i) \ \forall i=1,\ldots, d$, $N\in \N$, $p\in [0,1]$
 \Ensure $\x_N$, $y_N$
 \State Get $\w_k$, $b_k$, $k=1,\ldots D$ from Algorithm~\ref{alg:model1} for \red{IDONE-basic} or from Algorithm~\ref{alg:model2} for \red{IDONE-advanced}%, with the corresponding $D$.
 \State $\c_0 \leftarrow [0, 1, \ldots, 1]^T \in \R^D$
 %\State $P_0 \leftarrow \frac{1}{\lambda} I_{D\times D}$
\For{$n=1, \ldots, N$}
    \State Evaluate $y_n = f(\x_n) + \epsilon$
    \State Calculate $\c_n$ from $\c_{n-1}$ with recursive least squares 
    % \State $z_k(\x_n) \leftarrow  \w_k^T \x_n + b_k, \  k=1,\ldots, D$ \Comment{Update surrogate model}
    % \State $u(\x_n) \leftarrow \left[\max\{0,z_1(\x_n)\}, \ \ldots, \ \max\{0,z_D(\x_n)\}\right]^T$
    % \State $\alpha_n \leftarrow \frac{P_{n-1}u(\x_n)}{1+u(\x_n)^T P_{n-1} u(\x_n)}$
    % \State $P_n \leftarrow P_{n-1} - \alpha_n u(\x_n)^T P_{n-1}$
    % \State $\c_{n} \leftarrow \c_{n-1} + \alpha_n \left( y_n - u(\x_n)^T \c_{n-1}\right) $
    \State Compute $\x^*$ using \eqref{eq:minmodel}
    \If{$n<N$}
    \State $\x_{N+1} \leftarrow \x^* + \delta$, with $\delta$ as in Section~\ref{sec:exploration} 
    \EndIf
\EndFor
\end{algorithmic}
\end{algorithm}

\section{\red{Experimental} results}\label{sec:results}

\red{To determine how the IDONE algorithm compares to other black-box optimization algorithms in terms of convergence speed and scalability, it has been applied to two problems:
}
%
%
%The IDONE algorithm has been tested on \red{two} examples:
\red{finding a robust route for a noise-perturbed asymmetric traveling salesman benchmark problem with $17$ cities,
and an artificial convex binary optimization problem.
The first problem gives a first indication of the algorithm's performance on an objective function that follows from a simulation where there is a network structure.
The second problem shows an easier and more tangible situation - due to the convexity and the fact that we know the global optimum - which makes it easier to interpret results.}

\red{The algorithm is compared with several different black-box optimization algorithms: random search (RS), simulated annealing (SA), and two different Bayesian optimization algorithms: Matlab's \emph{bayesopt} function~\cite{bayesopt} (BO), and the Python library \emph{HyperOpt}~\cite{bergstra2013hyperopt} (HypOpt).
The RS, IDONE-advanced and HypOpt algorithms are implemented in Python and run on a cluster ($32$ Intel Xeon E5-2650 $2.0$ GHz CPUs), while the SA, BO and IDONE-basic algorithms are implemented in Matlab and run on a laptop (Intel Core i7-6600U $2.6$ GHz CPU with $8$ GB RAM), which is about $34\%$ faster%
\footnote{The cluster makes it faster to run multiple experiments at the same time, but no effort has been made to parallelize the algorithms. Matlab makes use of multiple cores automatically, and since the mentioned laptop has four cores, the computational results of the algorithms that were implemented in Matlab can be up to $5.4$ times faster than the python implementations, not just $34\%$.}
 according to  \url{https://www.cpubenchmark.net/singleThread.html}.
%Therefore, the computation times for the algorithms in this section should only be considered relatively to the problem size and relative to other algorithms that made use of the same computing resources.
Whenever we compare the runtimes of algorithms evaluated in these two different machine environments we will be more careful with our conclusions.
For BO and HypOpt, we used the default settings.
It should be noted that BO and HypOpt are both aimed at minimizing black-box functions using as few function evaluations as possible.
%The used settings for SA can be found in Appendix~\ref{sec:SA}.
For SA, the settings are explained below.
}

%\subsubsection{SA}
In the context of the IDONE algorithm, the SA algorithm essentially consists of just the exploration step of the IDONE algorithm (see Section~\ref{sec:exploration}), coupled with a probability of returning to the previous candidate solution.
Suppose the current best solution is $(\x_b, y_b)$, and that the exploration step as defined in Section~\ref{sec:exploration} gives a new candidate solution $(\x_c, y_c)$.
If $y_c<y_b$, then $\x_c$ is accepted as the new best solution.
Else, there is a probability that $\x_c$ is still accepted as the new best solution.
This probability is equal to $e^{(y_b-y_c)/T}$, with $T$ a so-called temperature. In this work, the simulated annealing algorithm starts out with a starting temperature $T=T_0$, and the temperature is multiplied with a factor $T_f$ every iteration.
This strategy is called a cooling schedule.
For the asymmetric traveling salesman problem, $T_0=4.48$ and $T_f=0.996$ have been chosen.
For the convex binary optimization problem, $T_0=1$ and $T_f=0.95$ have been chosen.

%a toy example, where a comparison is made with simulated annealing (SA), a simple traveling salesman problem with $4$ cities, and a noisy asymmetric traveling salesman benchmark problem with $17$ cities, where a comparison is made with SA and with Bayesian optimization.
% First, the algorithm is tested on a toy example with $150$ variables, and compared with simulated annealing (SA).
% Implementation details of SA are given in Appendix~\ref{sec:SA}.
% Second, the surrogate model is investigated in more detail for a simple traveling salesman problem with $4$ cities, mainly for visualization purposes.
% Third, IDONE, SA, and a Bayesian optimization algorithm are compared on a noisy asymmetric traveling salesman benchmark problem with $17$ cities.

\subsection{\red{Robust routes for an} asymmetric traveling salesman problem (17 cities)}

%As a third benchmark, a benchmark traveling salesman problem (TSP) with $17$ cities is considered. 
%The IDONE and SA algorithms, as well as a state-of-the-art Bayesian optimization (BO) algorithm, were tested on the asymmetric TSP called BR17.
\red{Consider the asymmetric TSP benchmark called BR17.}
This benchmark was taken from the TSPLIB website~\cite{TSPlib}, a library of sample instances for the traveling salesman problem.
%The BO algorithm that was used is the Matlab function \emph{bayesopt} using the default settings, except for choosing all variables to be integer.
\red{While there exist specific solvers developed for this problem, these solvers are not adequate if the objective to be minimized is perturbed by noise.}
\red{Here, }noise $\epsilon\in[0,\red{1}]$, with a uniform distribution, was added to the distances between any two cities
%\red{This severely distorted the problem, as 
\red{(for distances other than $0$ or infinity, which both occurred once per city, the mean distance between cities is $16.43$ for this instance)}.
\red{Furthermore, every time a sequence of cities has been chosen, we evaluate this route $100$ times, with different noise samples.
The objective is the worst-case (largest) route length of the chosen sequence of cities.
Minimizing this objective should then result in a route that is robust to the noise in the problem.
}
\red{For the variables} the same encoding as in Section~\ref{sec:TSP4} has been used, 
%That is, variable $x_1$ determines the starting city, $x_2$ determines which of the $16$ remaining cities is visited next, etc.
giving $15$ integer variables in total.

All algorithms were run $5$ times on this problem, and the results are shown in Figure~\ref{fig:TSP17}.
%The BO algorithm had to be cut off after $1250$ iterations due to the computation time: one run of BO took over $20$ hours for these $1250$ iterations, while IDONE1, IDONE2 and SA took $20$, $28$, and $18$ minutes respectively for $5000$ iterations.
\red{
The differences between the run times of the algorithms are significantly larger than between the two machine environments, so we ignore this subtlety here.
The BO algorithm was not included as it took over $80$ hours per run.}
It can be seen that \red{IDONE-advanced achieves similar results as HyperOpt,}
%the IDONE2 algorithm, which is the IDONE algorithm using model 2, 
%outperforming} SA,
\red{although its computation time is over ten times larger}.
\red{Both HyperOpt and IDONE-advanced outperform the simpler benchmark methods.
It seems IDONE-basic is unable to deal with the complex interaction between the variables due to the basic structure of the model, as it performs worse than the baseline algorithms.
}
%and also has a smaller variance.
%The BO algorithm is expected to achieve similar results, but the computation time is orders of magnitude larger. Overall, IDONE consistently produces high quality solutions within reasonable computation time.
%{\color{red}Any other conclusions? Does this really mean IDONE can be applied to practical problems?}
%\red{[UPDATE TEXT WITH NEW RESULTS]}

\begin{figure}[tb]
\begin{center}
\includegraphics[width=0.99\columnwidth]{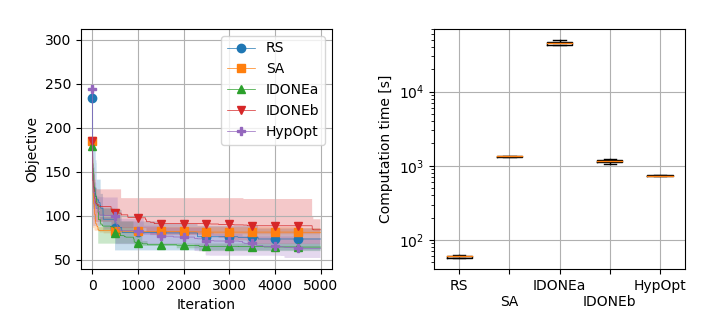}
\caption{Best found \red{worst-case} total distance \red{(\emph{left}) and corresponding computation time (\emph{right}) }of the \red{noisy} TSP with $17$ cities for IDONE-advanced (IDONEa), IDONE-basic (IDONEb), \red{random search (RS), simulated annealing (SA)}, and 
%Bayesian optimization (BO),
\red{HyperOpt (HypOpt)}, averaged over $5$ runs. 
%All $5$ runs lie within the shaded area.
The shaded area (\emph{left}) visualizes the range across all $5$ runs, as do the boxplots (\emph{right}).
%IDONE1 uses model 1, IDONE2 uses model 2.
%\red{[ADD TO PLOT: IDONEb]
} \label{fig:TSP17}%

\end{center}
\end{figure}

\subsection{\red{Convex binary optimization}}

\red{To 
%determine the scalability of the IDONE algorithm and 
gain a better understanding of the different algorithms, the second experiment is done on a function with a known mathematical formulation.}
Consider the function
\begin{align}
    f(\x) = (\x-\x^*)^T A (\x-\x^*),\label{eq:toyfunction}
\end{align}
with $A$ a random positive semi-definite matrix, and $\x^*\in\{0,1\}^{\red{d}}$ a randomly chosen vector\red{, with $d$ the number of binary variables}. % (see Section~\ref{sec:matrixA} for how $A$ is constructed).
\red{The optimal solution $\x^*$ or the structure of the function is not given to the different algorithms, only the number of variables and the fact that they are binary.
}
Starting from a matrix $U$ where each element is randomly generated from a uniform $[0,1]$ distribution, matrix $A$ is constructed as 
\begin{align}
    A &= (U + U^T)/\red{d} + I_{\red{d}\times \red{d}},
\end{align}
with $I$ the identity matrix.
The function $f$ can only be accessed via noisy measurements $y=f(\x)+\epsilon$, with $\epsilon\in[0,1]$ a uniform random variable.
We ran $100$ experiments with this function, with IDONE and \red{the other black-box optimization} algorithms.
For each run, $A$ and $\x^*$ were randomly generated, as well as the initial guess $\x_0$. %, and all three were shared between both algorithms.
\red{All} algorithms were stopped after taking $1000$ function evaluations, and the best found objective value was stored at each iteration.
%Figure~\ref{fig:toyexample} (\emph{Left}) shows the average over all the runs, from which it can be seen that the IDONE algorithm takes less function evaluations to converge. \red{[TO DO: ANALYSIS OF CONVERGENCE PLOT]}
\red{Figure~\ref{fig:toyexample} shows a convergence plot for the case $d=100$.
It can be seen that the two variants of IDONE
%~\footnote{
%IDONEs is a Matlab implementation of the IDONE algorithm using the simple model
%}
have the fastest convergence.
The large number of variables is too much for a pure random search, but also for HyperOpt.
Simulated annealing still gives decent results on this problem.
}

\red{Figure~\ref{fig:toyexample_differentdims} shows the final objective value and computation time after $1000$ iterations for the same problem for different values of $d$.
%, using the different algorithms.
%This time, the horizontal axis shows the dimension $d$ of the problem, which
The number of variables $d$ was varied between $5$ and $150$.
The Bayesian optimization implementation of Matlab was only evaluated for $d=5$ due to its large computation time.
As can be seen, IDONE (both versions) is the only algorithm that consistently gives a solution at or close to the optimal solution (which has an objective value between $0$ and $1$) for the highest dimensions.
Where all algorithms get at or close to the optimal solution for problems with $10$ variables or less, the difference between the algorithms becomes more distinguishable when $20$ or more variables are considered.
The strengths of HyperOpt, such as dealing with different types of variables that can have complex interactions, are not relevant for this particular problem, and the Parzen estimator surrogate model does not seem to scale well to higher dimensions compared to the piece-wise linear model used by IDONE.
%This does come at the cost of computation time compared to HyperOpt: for example, for a $100$-dimensional problem the computation time lies at around $0.8$ seconds per iteration.
%Any problem with an objective function that takes multiple tenths of seconds to evaluate can thus benefit from the much faster convergence of IDONE.
%It should be noted that unlike IDONE, the HyperOpt implementation has had multiple years of development, partly explaining the lower computation time. 
%It can also be seen that simpler methods such as simulated annealing or random search should not be discarded.
%In fact, these methods should be considered first due to their low computation time.
% [TO DO: WHY IS IDONEb BETTER THAN IDONEa.
% AND WHY ARE HYPEROPT RESULTS SO BAD.
% IT SEEMS LIKE FOR HIGHER DIMENSIONS, RANDOM WALKS (LIKE IN SA) ARE BETTER THAN COMPLETELY RANDOM GUESSES, SHOULD THIS BE DISCUSSED?]
}

% \red{Outdated: }Although the IDONE algorithm uses less function evaluations than SA, this comes at the cost of more computation time.
% This trade-off is only beneficial if each function evaluation is expensive. 
% This is illustrated in Figure~\ref{fig:5min} (\emph{Right}).
% It was calculated how many iterations could be performed with a fixed budget of $5$ minutes, if the function evaluation took longer.
% % The IDONE algorithm took around $0.3$ s per iteration, or $0.06$ s with Model 1.
% % The SA algorithm took less than $1$ ms per iteration, just like evaluating the objective.
% Since SA is so much faster than IDONE, it generally performs better if a fixed time budget is given and function evaluations are fast.
% However, if each function evaluation would take $500$ ms or more, the IDONE algorithm performs better than SA, since it uses fewer function evaluations to converge.
% \red{[UPDATE TEXT WITH NEW RESULTS]}

\begin{figure}[tbp]
\centering
\includegraphics[width=0.55\textwidth]{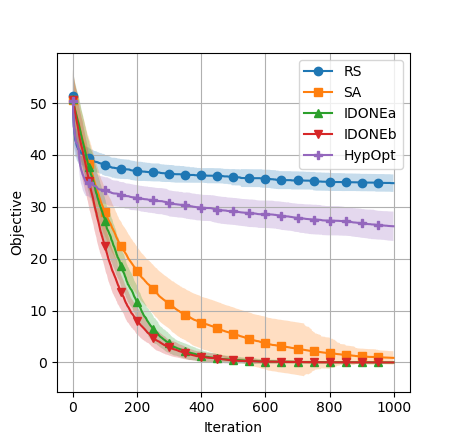}
\caption{
% \emph{Left}: Lowest objective value found at each iteration for IDONE on model 1 and 2 and simulated annealing (SA) on the \red{binary convex optimizaiton} example with $150$ binary variables, averaged over $100$ runs. For every run, the initial value, matrix $A$, and vector $\x^*$ were chosen randomly. The shaded area indicates the $5$ and $95$ percentiles.
% \emph{Right}: Expected objective value of the $150$-dimensional toy example after $5$ minutes, for different evaluation times of the objective.
% Lower values are better.
% If evaluating the objective is fast (right), SA performs better on a fixed time budget, but if evaluating the objective takes $500$ ms or more (left), SA performs worse because it needs more function evaluations to converge than the IDONE algorithm. 
\red{Lowest objective value found at each iteration of the binary convex optimization example with $100$ binary variables, averaged over $100$ runs.
The shaded area indicates the standard deviation.
For every run, the initial value, matrix $A$, and vector $\x^*$ were chosen randomly. 
}
} 
%\label{fig:5min}
\label{fig:toyexample}
\end{figure}

\begin{figure}[tbp]
\centering
\includegraphics[width=0.99\textwidth]{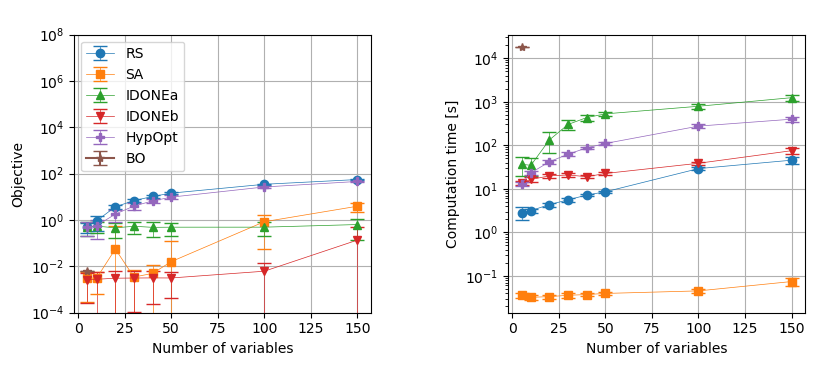}
\caption{\red{Objective value (\emph{left}) and computation time (\emph{right}) of the convex binary optimization problem for the different algorithms after $1000$ iterations, averaged over $100$ runs, for problems with different numbers of variables $d$.}}
\label{fig:toyexample_differentdims}
\end{figure}

\section{Conclusions and future work}\label{sec:conclusion}

The IDONE algorithm is a black-box optimization algorithm that
%is designed to be used on applications where 
is designed for combinatorial problems with binary or integer constraints, and had shown to be useful in particular when the objective 
can only be accessed via costly and noisy evaluations.
By using a surrogate model that is designed in such a way that its optimum lies in an integer solution, the algorithm can be applied to combinatorial optimization problems without having to resort to rounding in the objective function.
IDONE has a fixed computation time per iteration that scales quadratically with the number of variables but is not influenced by the number of times the function has been evaluated, which is an advantage compared to Bayesian optimization algorithms.
%This is shown on a noisy traveling salesman benchmark problem on which Bayesian optimization times out while IDONE consistently produces good solutions.
One variant of the proposed algorithm, IDONE-advanced, has been shown to outperform random search and simulated annealing on the problem of finding robust routes in a noise-perturbed traveling salesman benchmark problem, and on a convex binary optimization problem with up to $150$ variables.
The other variant of the algorithm, IDONE-basic, is a lot faster, but only performed well in the second experiment.
HyperOpt, a popular Bayesian optimization algorithm, performs similar as IDONE-advanced on the traveling salesman benchmark problem, but does not scale as well on the binary optimization problem.
These results show that there is room for improvement in the use of surrogate models for black-box combinatorial optimization, and that using continuous models with integer-valued local minima is a new and promising way forward.

%\nopagebreak
In future work, the special structure of the surrogate model will be further exploited to provide a faster implementation, and the algorithm will be tested on real-life applications of combinatorial optimization with expensive objective functions.
The question also arises whether this algorithm would perform well in situations where the objective function is not expensive to evaluate, or does not contain noise.
Population-based methods perform particularly well on cheap black-box objective functions, so it would be interesting to see if they could be combined with the surrogate model used in this paper.
As for the noiseless case, it is known that for continuous variables it becomes easy in this case to estimate the gradient and use more traditional gradient-based methods, but in the case of discrete variables the traditional combinatorial optimization methods might still benefit from IDONE's piece-wise linear surrogate model.
Where surrogate-based optimization techniques have had great success in continuous optimization problems from many different fields, we hope that this work opens up the route to success of these techniques for the plenty of open combinatorial problems in these fields.

\bibliographystyle{ieeetr}
\bibliography{mybib}

\end{document}